\documentclass{article}
\usepackage{ijcai22}

\usepackage{times}
\usepackage{soul}
\usepackage{url}
\usepackage[hidelinks]{hyperref}
\usepackage[utf8]{inputenc}
\usepackage[small]{caption}
\usepackage{graphicx}
\usepackage{amsmath}
\usepackage{amsthm}
\usepackage{booktabs}
\usepackage{algorithm}
\usepackage{algorithmic}
\usepackage{amssymb}
\usepackage{multirow}
\usepackage{color}

\usepackage{microtype}
\usepackage{subfigure}
\usepackage{booktabs} 
\usepackage{diagbox} 
\usepackage{multirow}
\usepackage{appendix}

\newtheorem{mydef}{Definition}
\newtheorem{thm}{Theorem}
\newtheorem{lem}{Lemma}

\title{Towards Safe Reinforcement Learning via Constraining Conditional Value-at-Risk}

\author{
Chengyang Ying$^1$
\and
Xinning Zhou$^1$\and
Hang Su$^{\ast 1,2,3}$\and
Dong Yan$^1$\and
Ning Chen$^1$\And
Jun Zhu\thanks{Corresponding author.} $^{1,2,3}$
\affiliations
 $^{1}$Department of Computer Science \& Technology, Institute for AI, BNRist Center,\\
Tsinghua-Bosch Joint ML Center, THBI Lab, Tsinghua University;\\
$^{2}$Peng Cheng Laboratory\\
$^{3}$ Tsinghua University-China Mobile
Communications Group Co., Ltd. Joint Institute\\
\emails
ycy21@mails.tsinghua.edu.cn,
\{coderlemon, sproblvem17\}@gmail.com,\\
\{suhangss, ningchen, dcszj\}@mail.tsinghua.edu.cn
}

\definecolor{mydarkblue}{rgb}{0,0.08,0.45}
\definecolor{mydarkgreen}{RGB}{0, 139, 69}
\hypersetup{
	colorlinks=true,
	urlcolor=magenta,
	citecolor=mydarkblue,
}

\makeatletter
\newtheorem*{rep@theorem}{\rep@title}
\newcommand{\newreptheorem}[2]{%
\newenvironment{rep#1}[1]{%
 \def\rep@title{#2 \ref{##1}}%
 \begin{rep@theorem}}%
 {\end{rep@theorem}}}
\makeatother
\newreptheorem{theorem}{Theorem}

\begin{document}
\maketitle
\begin{abstract}
  Though deep reinforcement learning (DRL) has obtained substantial success, it may encounter catastrophic failures due to the intrinsic uncertainty of both transition and observation. Most of the existing methods for safe reinforcement learning
  can only handle transition disturbance or observation disturbance since these two kinds of disturbance affect different parts of the agent;   besides, the popular worst-case return
  may lead to overly pessimistic policies. To address these issues, we first theoretically prove that the performance degradation under transition disturbance and observation disturbance depends on a novel metric of Value Function Range (VFR), which corresponds to the gap in the value function between the best state and the worst state. Based on the analysis, we adopt conditional value-at-risk (CVaR) as an assessment of risk and propose a novel reinforcement learning algorithm of CVaR-Proximal-Policy-Optimization (CPPO) which formalizes the risk-sensitive constrained optimization problem by keeping its CVaR under a given threshold.
 Experimental results show that CPPO achieves a higher cumulative reward and is more robust against both observation and transition disturbances on a series of continuous control tasks in MuJoCo. 
\end{abstract}


\section{Introduction}

Deep reinforcement learning (DRL) has achieved enormous success on a variety of tasks, ranging from playing Atari games
~\cite{atari_nature} and Go~\cite{go_nature} to manipulating complex robotics in real world~\cite{kendall2019learning}. However, due to the intrinsic uncertainty of both transition and observation, these methods may result in catastrophic failures
~\cite{Heger1994,sandy17}, i.e., the agent may receive significantly negative outcomes. This phenomenon is attributed to several factors. One is that traditional DRL only aims at cumulative return maximization without considering the stochasticity of the environment~\cite{survey15}, which may lead to serious consequences with a certain probability and thereby expose policies to risk. This can be illustrated briefly in the case of self-driving, where the agent might try to achieve the highest reward by acting dangerously, e.g., agents may drive along the edge of a curve in order to reach the destination more quickly without considering the potential danger. Also, a random disturbance or adversarial disturbance may interfere with the agent's observation, yielding a significant performance degeneration~\cite{sandy17}.

Various efforts has been made on safe reinforcement learning (safe RL) to handle transition uncertainty and observation uncertainty
~\cite{Heger1994,survey15,Huan2020}. ~\citeauthor{survey15} (\citeyear{survey15}) conduct a comprehensive survey on safe RL and argue that an array of methods focusing on transition uncertainty are based on transforming the optimization criterion. For example, robust approximate dynamic programming~\cite{Aviv2013}, based on a projected fixed point equation, considers how to solve robust MDPs~\cite{wiesemann2013robust} approximately to improve the robustness of the agent under transition uncertainty. Moreover, there is an array of work paying attention to observation disturbance. For example, some recent work formulates observation disturbance as a state-adversarial Markov decision process (SA-MDP) and proposes robust algorithms against observation disturbance. However, such work for handling transition uncertainty and observation uncertainty has some major drawbacks. First, due to the consideration of the worst-case outcomes~\cite{Heger1994}, those methods 
may lead to overly pessimistic policies, which will focus too much on the worst case and own poor average performance. Moreover, though the agent may suffer from transition uncertainty and observation uncertainty at the same time, existing work~\cite{Arnab2005,Aviv2013,Huan2020} can only handle observation disturbance or transition disturbance separately. The main reason is that these two kinds of disturbance are structurally different. To the best of our knowledge, there is currently no analysis on the connection of the two typical uncertainties and few methods to deal with them both at the same time.




To build a connection between transition disturbance and observation disturbance, we first prove that the performance degradation resulting from each of them is theoretically dependent on a new notion of \emph{Value Function Range} (VFR), 
which is the value function gap between the best and worst states. However, directly controlling VFR may also suffer from the problem of excessive pessimism, because VFR only considers the value of extreme states, and moreover, the value function calculated in VFR is difficult to estimate. We first use conditional value-at-risk (CVaR) as a slack of the minimum since CVaR can be used for avoiding overly pessimistic policies~\cite{alexander2004comparison,alexander2006minimizing}. Moreover, We theoretically prove that the CVaR of the return of trajectories is a lower bound of the CVaR of the value function and explain that the former one is much easier to estimate.  

Based on this theoretical analysis, 
we propose to use the CVaR of the return of trajectories to replace VFR and formulate a CVaR-based constrained optimization problem for safe RL under transition disturbance as well as observation disturbance. Furthermore, we analyze the properties of this optimization problem and present a new algorithm called CVaR-Proximal-Policy-Optimization (CPPO) based on Proximal Policy Optimization (PPO)~\cite{ppo}. 
Empirically, we compare CPPO to multiple on-policy baselines as well as some previous CVaR-based methods on various continuous control tasks in MuJoCo~\cite{mujoco}. Our results show that CPPO achieves competitive cumulative reward in the training stage 
and exhibits stronger robustness when we apply perturbations to these environments. 


In summary, our contributions are: 
\begin{itemize}
    \item We theoretically analyze the performance of trained policies under transition and observation disturbance, and use VFR to build a theoretical connection between these two types of structurally different disturbance; 
    \item Based on the analysis, we present a constrained optimization problem to maximize the cumulative reward and simultaneously control the risk, which is solved by our CPPO algorithm under the regularization of CVaR;
    \item We empirically demonstrate that CPPO exhibits stronger robustness under transition/observation perturbations compared with the alternative common on-policy RL algorithms and previous CVaR-based RL algorithms on different MuJoCo tasks. 
\end{itemize}

\section{Background}
In this section, we briefly introduce safe reinforcement learning (safe RL) and conditional value-at-risk (CVaR), which motivate us to adopt CVaR as a metric of risk in safe RL.

\subsection{Safe RL}
\label{safeRL_section}
In a standard RL setting, the agent interacts with an unknown environment and learns to achieve the highest long-term return. The task is modeled as a Markov decision process (MDP) \label{MDP} of  $\mathcal{M}=(\mathcal{S},\mathcal{A}, \mathcal{R}, \mathcal{P}, \gamma)$, where $\mathcal{S}$ and $\mathcal{A}$ represent the state space and the action space, respectively; $\mathcal{P}: \mathcal{S}\times \mathcal{A}\times \mathcal{S}\rightarrow [0,1]$ denotes the transition probability that captures the dynamics of the environment;  $\mathcal{R}: \mathcal{S}\times \mathcal{A} \rightarrow [-R_{\max},R_{\max}]$ represents the reward function; and $\gamma$ is a discount factor. We use $\pi_{\theta}$ to represent the policy of the agent with parameter $\theta$, which is a mapping from $\mathcal{S}$ to the set of distributions on $\mathcal{A}$. At each time step $t$, the agent perceives the current state $s_t\in\mathcal{S}$, chooses its action $a_t \in \mathcal{A}$ sampled from the distribution $\pi_\theta(\cdot|s_t)$ and obtains a reward $r_t$. 
However, most of the current algorithms try to maximize cumulative reward without considering the risk of the policy, which may cause catastrophic results
~\cite{Heger1994}. 

To address this problem, an array of safe RL methods tend to change the objective in order to eliminate the uncertainty and avoid the danger~\cite{survey15}. In general, two kinds of uncertainty are widely discussed, namely, transition uncertainty and observation uncertainty. The transition uncertainty of RL denotes scenarios where the parameters of the MDP are unknown or there is a gap between the training and testing environments. Studies conducted by~\citeauthor{Arnab2005} and~\citeauthor{Aviv2013} assume that the actual transition belongs to a set $\hat{\mathcal{P}}$ and propose to optimize
\begin{equation}
\begin{aligned}
\label{parameter-uncertainty}
    \max_{\theta}\min_{\mathcal{P}\in\hat{\mathcal{P}} } J_{\mathrm{tr}}(\pi_{\theta}, \mathcal{P}) \triangleq \mathbb{E}\left[D(\pi_{\theta}) \triangleq \sum_{t=1}^{\infty} \gamma^t r_{t}\Big{|}\pi_{\theta}, \mathcal{P}\right].
\end{aligned}
\end{equation}

As for the observation uncertainty of RL, it refers to the gap between the observation and the true state. For example, the observation of the agent may be disturbed by random disturbance as well as adversarial disturbance, which will cause a drop in the performance~\cite{sandy17}. For evaluating observation uncertainty, some previous work~\cite{Huan2020} has assumed that the observation will be disturbed as $\nu(s)\in\mathcal{S}$ when the true state is $s$, and hope to find a robust policy under any $\nu\in\Gamma$, here $\Gamma$ is the set of all state-observation disturbance. Based on that assumption, such work has built a framework named state-adversarial MDP (SA-MDP) to solve
\begin{equation}
\label{observation-uncertainty}
    \max_{\theta}\min_{\nu\in\Gamma}J_{\mathrm{obs}}(\pi_{\theta}) \triangleq \mathbb{E}\left[D(\pi_{\theta}, \nu) \triangleq \sum_{t=1}^{\infty} \gamma^t r_{t} \right].
\end{equation}
However, existing safe RL methods are not problemless. First, both (\ref{parameter-uncertainty}) and (\ref{observation-uncertainty}) are \textit{max-min problems}, which do not have general effective solvers and usually have high computational complexity. Second, focusing on the worst trajectories may cause overly pessimistic behaviors. 
Finally, because these two kinds of uncertainty are structurally different, existing work always considers them separately rather than building a connection between them. 

\subsection{CVaR}
\label{CVaR_section}
Value-at-risk (VaR) and conditional value-at-risk (CVaR) are both well-established metrics for measuring risk in economy
~\cite{alexander2004comparison}. First, we give their definitions~\cite{CVaRnips} below.
\begin{mydef}[VaR and CVaR]
\label{CVaR}
For a bounded-mean random variable $Z$, the value-at-risk (VaR) of $Z$ with confidence level $\alpha\in(0,1)$ is defined as:
\begin{equation}
    \mathrm{VaR}_{\alpha}(Z) = \min \{z|F(z)\ge \alpha\},
\end{equation}
where $F(z) = P(Z \leq z)$ is the cumulative distribution function (CDF); 
and the conditional value-at-risk (CVaR) of $Z$ with confidence level $\alpha$ is defined as the expectation of the $\alpha$-tail distribution of $Z$ as
\begin{equation}
    \mathrm{CVaR}_{\alpha}(Z) = \mathbb{E}_{z\sim Z}\{z|z\ge \mathrm{VaR}_{\alpha}(Z)\}.  
\end{equation}
\end{mydef}
It is easy to prove that \cite{CVaRnips2}
\begin{equation}
\label{CVaR_prop1}
    \lim_{\alpha\rightarrow 1^-}\mathrm{CVaR}_{\alpha}(Z) = \max(Z).
\end{equation}


Previous work~\cite{CVaRnips,CVaRnips2,CVaRml} has attempted to use CVaR to analyze the risk-MDP, which considers a risk function $\mathcal{C}$ rather than a reward function $\mathcal{R}$. They propose gradient-based methods and value-based methods to optimize loss of MDP as well as keeping the CVaR under a certain value. However, these studies ignore the reward in MDP and thus cannot be directly used in RL settings. 

Besides, there are an array of work for optimizing CVaR~\cite{tamar2015optimizing,tang2020worst}, optimizing the CVaR-constrained objective~\cite{prashanth2014policy,yang2021wcsac}, and analyzing the connection between optimizing CVaR and the robustness against transition disturbance~\cite{CVaRnips2,rigter2021risk}. Also, there are some work~\cite{ma2020dsac} extending methods in distributional RL~\cite{dabney2018implicit}, which mainly considers the randomness of the return, for CVaR optimization.




\section{Theoretical Analysis}
\label{sec-theo}
In this section, we first analyze the robustness of policies against transition perturbations and observation disturbances, and further build a connection between them.

\subsection{Value Function Range}
For an MDP $\mathcal{M}$ and a given policy $\pi$, we denote its expected cumulative reward and value function as $J_{\mathcal{M}}(\pi)$ and $V_{ \mathcal{M}, \pi}$~\cite{sutton2018reinforcement}, respectively. We define the {\it Value Function Range} (VFR) to capture the gap of the value function between the best state and the worst state as follows.
\begin{mydef}[Value Function Range]
For MDP $\mathcal{M}$, the Value Function Range (VFR) of the policy $\pi$ is
\begin{equation}
    \hat{V}_{\mathcal{M}, \pi} \triangleq \max_s V_{\mathcal{M}, \pi}(s) - \min_s V_{\mathcal{M}, \pi}(s).
\end{equation}
\end{mydef}
Moreover, for every state $s\in \mathcal{M}$, we define its discounted future state distribution as
\begin{equation*}
    d_{\mathcal{M}}^{\pi}(s) = (1-\gamma)\sum_{t=0}^{\infty} \gamma^t \mathcal{P}(s_t=s|\pi,\mathcal{M}).
\end{equation*} 

\subsection{Performance against Transition Disturbance}
First, we consider transition disturbance. Assume that the transition $\mathcal{P}$ is disturbed by $\hat{\mathcal{P}}$, we attempt to evaluate the reduction of cumulative reward against the disturbance. We can calculate and bound the difference of performance of $\pi$ under $\mathcal{M}$ and $\hat{\mathcal{M}}$ in Theorem \ref{trans_disturb} as below:
\begin{thm}
\label{trans_disturb}
For any policy $\pi$ in MDP $\mathcal{M} = (\mathcal{S}, \mathcal{A}, \mathcal{P}, \mathcal{R}, \gamma)$ and any disturbed environment $\hat{\mathcal{M}} = (\mathcal{S}, \mathcal{A}, \hat{\mathcal{P}}, \mathcal{R}, \gamma)$, the reduction of the cumulative reward against the transition disturbance is
\begin{equation*}
\begin{aligned}
    & J_{\hat{\mathcal{M}}}(\pi) - J_{\mathcal{M}}(\pi)
    \\
    = &\frac{\gamma}{1-\gamma} \mathbb{E}_{s\sim d_{\hat{\mathcal{M}}}^{\pi}} \mathbb{E}_{a\sim\pi} \mathbb{E}_{s'\sim \hat{\mathcal{P}}} \left(1 - \frac{\mathcal{P}(s'|s, a)}{\hat{\mathcal{P}}(s'|s, a)}\right)V_{\mathcal{M}, \pi}(s'). 
\end{aligned}
\end{equation*}
Furthermore, an upper bound of the reduction is
\begin{equation}
\begin{split}
    &\left|J_{\mathcal{M}}(\pi) - J_{\hat{\mathcal{M}}}(\pi)\right|
    \\
    \leq &\frac{2\gamma}{1-\gamma} \max_{s,a} D_{\mathrm{TV}}(\mathcal{P}(\cdot|s,a), \hat{\mathcal{P}}(\cdot|s,a)) \hat{V}_{\mathcal{M}, \pi}.
\end{split}
\end{equation}
\end{thm}
The key of the proof is to analyze the relationship of $V_{\mathcal{M}, \pi}-V_{\hat{\mathcal{M}}, \pi}$ with different states. We defer the complete proof to Appendix~\ref{thm_pf_3_4}, 
which resembles the proof by~\cite{kakade2002approximately}. Compared with the policy $\pi$ for a given MDP $\mathcal{M}$, the factors that mainly affect the performance of $\pi$ in disturbed environment $\hat{\mathcal{M}}$ are Total Variation (TV) distance $\max_{s,a} D_{\mathrm{TV}}(\mathcal{P}(\cdot|s,a), \hat{\mathcal{P}}(\cdot|s,a))$ and VFR $\hat{V}_{\mathcal{M}, \pi}$. The TV distance, depends on the range of transition disturbance, 
which is independent with the agent and cannot be controlled by safe RL. By contrast, VFR depends only on the value functions of $\pi$ in $\mathcal{M}$ and is an intrinsic property of the policy $\pi$. Therefore, we can improve the robustness of the policy $\pi$ under a transition disturbance policy by controlling $\hat{V}_{\mathcal{M}, \pi}$.

\subsection{Performance against Observation Disturbance}
Now, we consider the situation of observation disturbance. Similarly to the setting of SA-MDP~\cite{Huan2020}, we introduce adversary $\nu: \mathcal{S}\rightarrow \mathcal{S}$ to describe the disturbance of the state and denote the policy disturbed by adversary $\nu$ as $\hat\pi_{\nu}$, which means $\hat\pi_{\nu}(\cdot|s) = \pi(\cdot|\nu(s))$. 
Similarly to Theorem \ref{trans_disturb}, we can also show a similar result as below:
\begin{thm}
\label{state_disturb}
For any policy $\pi$ and any adversary $\nu$, the reduction of the expected cumulative reward of $\pi$ against the observation disturbance of $\nu$ is
\begin{equation*}
\begin{split}
    &J_{\mathcal{M}}(\hat\pi_{\nu}) - J_{\mathcal{M}}(\pi)\\
    =& \frac{\gamma}{1-\gamma} \mathbb{E}_{s\sim d_{\mathcal{M}}^{\hat\pi_{\nu}}} \mathbb{E}_{a\sim\pi(\cdot|\nu(s))} \left(1- \frac{\pi(a|s)}{\pi(a|\nu(s))}\right)\mathbb{E}_{s'\sim \mathcal{P}} V_{\mathcal{M}, \pi}(s')\\
    +& \frac{1}{1-\gamma}\mathbb{E}_{s\sim d_{\mathcal{M}}^{\hat\pi_{\nu}}} \mathbb{E}_{a\sim\pi(\cdot|\nu(s))} \left(1- \frac{\pi(a|s)}{\pi(a|\nu(s))}\right)\mathcal{R}(s, a).
\end{split}
\end{equation*}
Furthermore, an upper bound of the reduction is
\begin{equation}
\label{state_bound_eq}
\begin{split}
     &|J_{\mathcal{M}}(\pi) - J_{\mathcal{M}}(\hat\pi_{\nu})| \\
     \leq &\frac{\gamma}{1-\gamma} \max_s D_{\mathrm{TV}}(\pi(\cdot|s), \pi(\cdot|\nu(s))) \hat{V}_{\mathcal{M}, \pi}\\
    +& \frac{2}{1-\gamma}\max_s D_{\mathrm{TV}}(\pi(\cdot|s), \pi(\cdot|\nu(s)) \max_{s,a}|\mathcal{R}(s, a)|.
\end{split}
\end{equation}
\end{thm}
The proof is similar to that of Theorem \ref{trans_disturb} and is also included in Appendix~\ref{thm_pf_3_4}. 
Moreover, for the upper bound, Theorem \ref{state_disturb} provides a bound that is structurally homologous to, but tighter than, the bound provided in~\cite{Huan2020}. This is because our VFR can be bounded by $\max_{s,a}|R(s, a)|$, which is also proven in Appendix~\ref{thm_pf_3_4}. 
Similarly, compared with the policy $\pi$ for the given MDP $\mathcal{M}$, the factors mainly affecting the performance of the disturbed policy $\pi_{\nu}$ are TV distance $\max_s D_{\mathrm{TV}}(\pi(\cdot|s), \pi(\cdot|\nu(s))$ and the VFR $\hat{V}_{\mathcal{M}, \pi}$. The TV distance, depends on the policy $\pi$ as well as the disturbance $\nu$, reflecting both the robustness of the policy $\pi$ and the adversarial ability. However, independent of the adversary, the latter factor (the VFR of the policy), only depends on the value functions of $\pi$ in $\mathcal{M}$, reflecting the robustness of the policy $\pi$. Thus, we can also improve the robustness under observation disturbance of the policy by controlling the VFR of the policy.

\subsection{Connection between the Transition and Observation Disturbance}
Transition disturbance and observation disturbance are structurally different, as they affect MDP and the observation of the policy respectively. Although existing literature usually considers them separately, by Theorem~\ref{trans_disturb} and Theorem~\ref{state_disturb}, we can find out that their effects on cumulative reward are similar; the similarity depends on the VFR $\hat{V}_{\mathcal{M}, \pi}$, which is an inherent property of $\pi$ and independent of the adversary. Theoretically, if we set $\epsilon_{\mathcal{P}} = \max_{s,a}D_{\mathrm{TV}}(\mathcal{P}(\cdot|s,a), \hat{\mathcal{P}}(\cdot|s,a))$, $\epsilon_{\pi} = \max_s D_{\mathrm{TV}}(\pi(\cdot|s), \pi(\cdot|\nu(s))$ and assume that $\max_{s,a}|\mathcal{R}(s,a)|=1$, we can naturally deduce
\begin{equation}
\begin{split}
    &\left|J_{\mathcal{M}}(\pi) - J_{\hat{\mathcal{M}}}(\pi)\right|
    \leq \frac{2\gamma}{1-\gamma} \epsilon_{\mathcal{P}} \hat{V}_{\mathcal{M}, \pi}\\
    &|J_{\mathcal{M}}(\pi) - J_{\mathcal{M}}(\hat\pi_{\nu})| 
     \leq \frac{\gamma}{1-\gamma} \epsilon_{\pi} \hat{V}_{\mathcal{M}, \pi} + \frac{2}{1-\gamma}\epsilon_{\pi}.
\end{split}
\end{equation}
Thus we can improve the robustness of the policy under observation and transition disturbances by controlling $\hat{V}_{\mathcal{M}, \pi}$.

\section{Methodology}
\label{sec-method}
In this section, we first formulate our problem for improving the robustness of the agent and then propose a novel on-policy algorithm of CPPO to solve it.



\subsection{Problem Formulation}
\label{sec-3-1}
We first discuss the connection between controlling the VFR $\hat{V}_{\mathcal{M}, \pi}$ and CVaR-based RL. For controlling $\hat{V}_{\mathcal{M}, \pi}$, it is more reasonable to maximize $\min_s V_{\mathcal{M}, \pi}(s)$ rather than minimize $\max_s V_{\mathcal{M}, \pi}(s)$ since the latter one contradicts our goal of maximizing cumulative expected return. 
However, as mentioned in Sec.~\ref{safeRL_section}, directly maximizing the value function of the worst state may cause our policy to be overly conservative. By the property (\ref{CVaR_prop1}) of CVaR, it is more reasonable to loosen $\min_s V_{\mathcal{M}, \pi}(s)$ to $-\mathrm{CVaR}_{\alpha}(-V(s))$ where $s\sim\mu(\cdot)$ obeys the initial distribution of the environment. Unfortunately, we cannot precisely approximate the value function of every state in practice. Since the return of every trajectory can be calculated exactly, we consider loosening $-\mathrm{CVaR}_{\alpha}(-V(s))$ by $-\mathrm{CVaR}_{\alpha}(-D(\pi))$ via Theorem~\ref{thm-5}.
\begin{thm}[Proof in Appendix~\ref{thm_pf_5}]
\label{thm-5}
For any $\alpha\in[0,1]$, we have
\begin{equation}
    -\mathrm{CVaR}_{\alpha}(-D(\pi))\leq -\mathrm{CVaR}_{\alpha}(-V(s)).
\end{equation}
\end{thm}
Therefore, we consider constraining $-\mathrm{CVaR}_{\alpha}(-D(\pi))$ to improve the VFR of the policy and further improve the robustness of the policy against observation disturbance as well as transition disturbance. Based on this analysis, we define our constrained optimization problem as 
\begin{equation}
\begin{aligned}
\label{raw_problem}
&\max_{\theta} J(\pi_{\theta})\quad s.t.-\mathrm{CVaR}_{\alpha}(-D(\pi_{\theta}))\ge \beta ,
\end{aligned}
\end{equation}
where $\alpha,\beta$ are hyper-parameters. 

We denote the best policy of problem (\ref{raw_problem}) as $\pi_c(\alpha,\beta)$. 
Compared with the best policy $\pi_s$ of the standard RL problem, we obviously have $J(\pi_c(\alpha,\beta)) \leq J(\pi_s)$ since $\pi_c(\alpha,\beta)$ is in a restricted region related to hyper-parameters $\alpha,\beta$. We can further give a lower bound of $J(\pi_c(\alpha,\beta))$ as follows:
\begin{thm}[Proof in Appendix~\ref{proof_lower_bound}]
\label{lower_bound}
Assume that the discounted return of every trajectory $\tau = (S_0,A_0,R_0,S_1,...)$ can be bounded by a constant $M$, i.e.,
$
    \sum_{t=0}^{\infty}\gamma^tR_t\leq M,
$
then we have
\begin{equation*}
    J(\pi_c(\alpha,\beta)) \ge \frac{J(\pi_s) - \alpha M}{1-\alpha}.
\end{equation*}
\end{thm}
By Theorem~\ref{lower_bound}, we can see that the expected cumulative return of $\pi_c(\alpha,\beta)$ will be no worse than the lower bound although it is optimized in a restricted region.

\begin{figure*}[t]
\centering
\includegraphics[height=4.2cm,width=16cm]{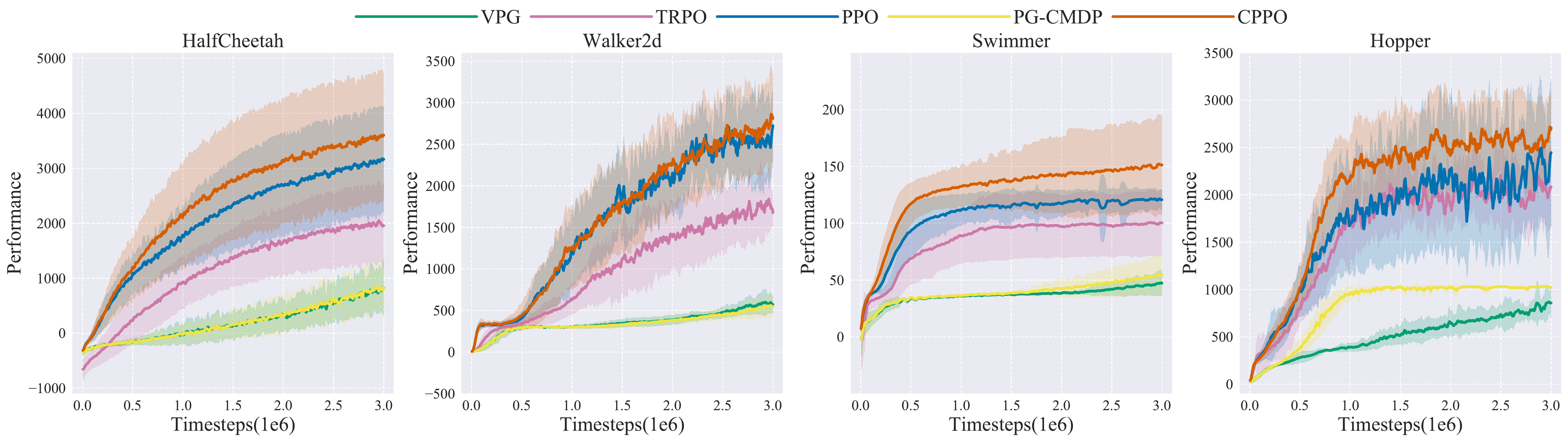}
\caption{Cumulative reward curves for VPG, TRPO, PPO, PG-CMDP and our CPPO. The x-axes indicate the number of steps interacting with the environment, and the y-axes indicate the performance of the agent, including average rewards with standard deviations.}
\label{fig_performance}
\end{figure*}

\begin{table*}[tbp]
\centering
\footnotesize
\begin{tabular}{cccccc}
\toprule
Method&Ant-v3&HalfCheetah-v3&Walker2d-v3&Swimmer-v3&Hopper-v3\\ 
\midrule
VPG 
&12.8$\pm$ 0.0
&896.9$\pm$ 531.1			
&628.6$\pm$ 229.4
&48.3$\pm$ 11.3
&888.4$\pm$ 209.5\\
TRPO
&1625.4$\pm$ 356.4		
&2073.8$\pm$ 741.3
&2005.6$\pm$ 398.7
&101.2$\pm$ 29.3
&2391.4$\pm$ 455.3\\
PPO
&3372.2$\pm$ 301.4		
&3245.4$\pm$ 947.3
&2946.3$\pm$ 944.3
&122.0$\pm$ 7.9
&2726.0$\pm$ 886.0\\
PG-CMDP
&7.4 $\pm$ 3.6		
&928.7$\pm$ 562.9
&596.7$\pm$ 219.9
&55.4$\pm$ 18.8
&1039.2$\pm$ 21.1\\
CPPO(ours) 
&\textbf{3514.7$\pm$ 247.2} 
&\textbf{3680.5$\pm$ 1121.3} &\textbf{3194.0$\pm$ 648.2} 
&\textbf{182.5$\pm$ 46.0} 
&\textbf{3144.6$\pm$ 158.4}\\
\bottomrule
\end{tabular}
\caption{Cumulative reward (mean $\pm$ one std) of best policy trained by VPG, TRPO, PPO, PG-CMDP and CPPO in different MuJoCo games. In each column, we \textbf{bold} the best performance over all algorithms.}
\label{performance_table}
\end{table*}

\begin{algorithm}[t]
    \caption{CVaR Proximal Policy Optimization (CPPO)}  
    \label{CPPO}
    \begin{algorithmic}[1] 
        \REQUIRE confidence level $\alpha$, learning rate $lr_{\eta}, lr_{\theta}, lr_{\lambda}, lr_{\phi}$
        \ENSURE parameterized policy $\pi_{\theta}$ and parameterized value function $V_{\phi}$.
        \FOR{$k = 1,2,..., N_{iter}$}
        \STATE 
        Generate $N$ trajectories with the current policy $\pi_{\theta}$.
        \STATE
        Compute advantage estimates $\hat A_i^t$ of each state $s_{i,t}$ in each trajectory $\xi_i$ and the cumulative reward $D(\xi_i)$.
        \STATE 
        Update parameters $\eta,\theta,\lambda,\phi$ respectively with the calculated gradients.
        \STATE Modify $\beta$ as a function of current trajectories' return.
        \ENDFOR
    \end{algorithmic}  
\end{algorithm} 

\subsection{Optimization and Algorithm}
We now simplify the constrained problem (\ref{raw_problem}) to an unconstrained one. First, with the properties of CVaR, we can equivalently reformulate problem (\ref{raw_problem}) as 
\begin{equation*}
\label{deform_problem}
\begin{split}
& \min_{\theta,\eta} -J(\pi_{\theta})\quad s.t.\;  \frac{1}{1-\alpha}\mathbb{E}[(\eta-D(\pi_{\theta}))^+] -\eta\leq -\beta.
\end{split}
\end{equation*}
The deviation is provided in Appendix~\ref{proof_deform_problem}.  
Moreover, by using a Lagrangian relaxation method~\cite{nonlinear},
we need to solve the saddle point of the function $L(\theta,\eta,\lambda)$ as
\begin{equation}
\begin{split}
\label{optimization_problem}
&\max_{\lambda\ge 0}\min_{\theta,\eta}  L(\theta,\eta,\lambda) \\
\triangleq &-J(\pi_{\theta})+ \lambda \left(\frac{1}{1-\alpha}\mathbb{E}[(\eta-D(\pi_{\theta}))^+] -\eta + \beta\right).
\end{split}
\end{equation}
To solve problem (\ref{optimization_problem}), we extend Proximal Policy Optimization (PPO)~\cite{ppo} with CVaR and name our algorithm CVaR Proximal Policy Optimization~(CPPO). In particular, the key point is to calculate gradients~\cite{PG}. Here, we use methods in~\cite{CVaRnips} to compute the gradient of our objective function (\ref{optimization_problem}) with respect to $\eta, \theta,\lambda$ as
\begin{equation*}
\nabla_{\eta}L(\theta,\eta,\lambda) =  \frac{\lambda}{1-\alpha}\mathbb{E}_{\xi\sim\pi_{\theta}}\textbf{1}\{\eta\ge D(\xi)\}) -\lambda
\end{equation*}
\begin{equation*}
\begin{split}
&\nabla_{\theta}L(\theta,\eta,\lambda)\\
=& -\mathbb{E}_{\xi\sim\pi_{\theta}}(\nabla_{\theta}\log P_{\theta}(\xi))\left(D(\xi) 
-\frac{\lambda}{1-\alpha}(-D(\xi)+\eta)^+ \right)
\end{split}
\end{equation*}
\begin{equation*}
\nabla_{\lambda}L(\theta,\eta,\lambda) =  \frac{1}{1-\alpha}\mathbb{E}_{\xi\sim\pi_{\theta}} (-D(\xi)+\eta)^++ \beta -\eta.
\end{equation*}
The detailed calculation is in Appendix B.5. 
Moreover, with the improvement of policies' performance during training, it is unreasonable to fix $\beta$ to constrain the risk of the policy. Thus, we consider modifying $\beta$ as a function of the risk of trajectories in the current epoch. For example, in CPPO, we set $\beta$ as the mean value of the expected cumulative return of the worst $K$ trajectories of all $N$ trajectories in the previous epoch, and we will set the ratio $K/N$ to be larger than the ratio in the constraint for reducing the risk. 
Algorithm \ref{CPPO} outlines the CPPO algorithm, and a more detailed version is in Appendix~\ref{pseudecode}.

\begin{figure*}[t]
\centering
\includegraphics[height=4.2cm,width=16cm]{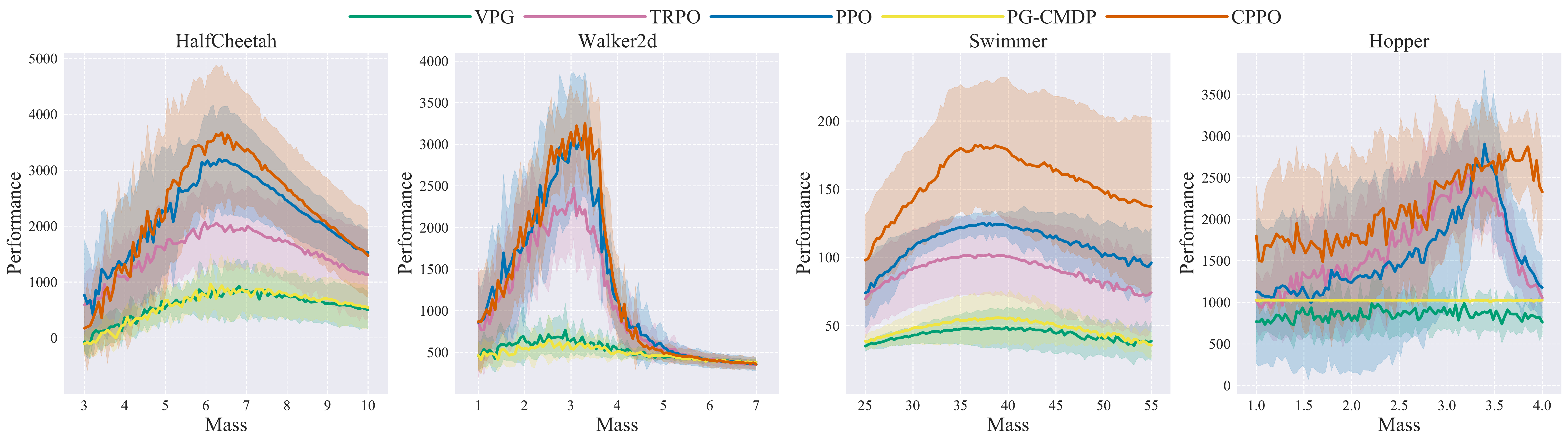}
\caption{Cumulative reward curves for VPG, TRPO, PPO, PG-CMDP and our CPPO under transition disturbance. The x-axes indicate the mass of the agent, and the y-axes indicate the average performance of the algorithm when the mass changes.}
\label{fig_mass}
\end{figure*}

\begin{figure*}[t]
\centering
\includegraphics[height=4.2cm,width=16cm]{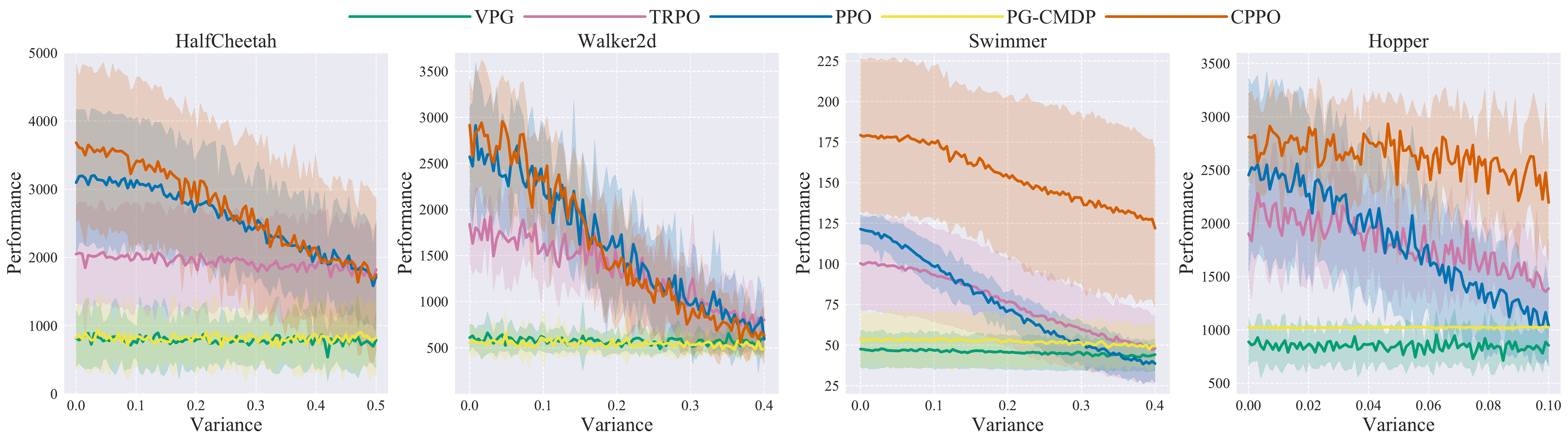}
\caption{Cumulative reward curves for VPG, TRPO, PPO, PG-CMDP and our CPPO under observation disturbance. The x-axes indicate the range of the disturbance, and the y-axes indicate the average performance of the algorithm under the state disturbance.}
\label{fig_ob}
\end{figure*}

\section{Experiments}
\label{sec-5}
In this section, we empirically evaluate the performance and robustness of CPPO under both transition disturbance and observation disturbance in a series of continuous control tasks in MuJoCo~\cite{mujoco} against other common on-policy RL algorithms. 

\subsection{Experiment Setup}
\label{experiment_setup}

\textbf{Environments.} We choose MuJoCo \cite{mujoco} as our experimental environment. As a robotic locomotion simulator, MuJoCo has lots of different continuous control tasks like Ant, HalfCheetah, Walker2d, Swimmer and Hopper, which are widely used for the evaluation of RL algorithms.

\textbf{Baselines and Codes.} We compare our algorithm with common on-policy algorithms and previous CVaR-based algorithms. For the former, we choose Vanilla Policy Gradient (VPG)~\cite{PG}, Trust Region Policy Optimization (TRPO)~\cite{trpo} and PPO~\cite{ppo}. For the latter, we implement PG-CMDP~\cite{CVaRnips} with a deep neural network. We use Adam~\cite{adam} to optimize all algorithms. The implementation of all code, including CPPO and baselines, are based on the codebase SpinningUp~\cite{SpinningUp2018}.


\textbf{Evaluation.} First, we compare the cumulative reward of each algorithm in the training process and its performance after convergence. In order to measure the robustness and safety, we compare the performance under transition disturbance and observation disturbance, respectively. For transition perturbation, since MuJoCo is a physical simulation engine and its transition depends on its physical parameters, we choose to modify the mass of the agent to change the transition dynamics, and study the relationship between the agent's performance and the mass of the agent. For observation disturbance, we apply Gaussian disturbance to the agent's observation to study the relationship between the agent's performance and the magnitude of the disturbance.

\subsection{Performance in the Training Stage}
\label{experiment_per}

In this part, we compare the performance in the training stage of our CPPO against common on-policy algorithms as well as the previous CVaR-based algorithm in MuJoCo environments. For each algorithm in each task, we train 10 policies with different random seeds, since the environments and policies are stochastic. 
For each algorithm in each task, we also plot the mean and variance of the 10 policies as a function of timestep in the training stage, as shown in Figure \ref{fig_performance}. The solid line represents the average reward of 10 strategies, and the part with a lighter color represents their variance. The final mean and variance of the cumulative return of 10 policies trained by each algorithm in each environment are reported in Table \ref{performance_table}. As shown in the figures and the table, for all five tasks, CPPO learns a better policy compared with all baselines even when there is no transition or observation disturbance, especially in HalfCheetah, Swimmer and Hopper. Compared with VPG, TRPO and PG-CMDP, CPPO performs better since we use better policy optimization techniques. Moreover, the performance of CPPO exceeds PPO since penalizing trajectories with relatively low return can also benefit the cumulative return. 

\subsection{Robustness against Transition Disturbance}
\label{experiment_tra}

An agent may fail in the testing stage because of the transition gap between the simulator and the true environment. In this section, we choose to modify the mass of the robot and test the performance of agents with different transitions, i.e., we change the default mass of HalfCheetah (6.36), Walker2d (3.53), Swimmer (34.6) and Hopper (3.53). Then, we draw Figure \ref{fig_mass} to describe the results of agents that are trained under standard mass conditions and tested under different mass conditions. The solid line represents the average reward of 10 strategies, and the part with a lighter color represents their variance. As seen in this figure, the performance of all algorithms decreases to a certain extent with the change of agent quality (whether it becomes larger or smaller). The degree of the decrease is positively correlated with the quality change, which is consistent with our theoretical analysis in Theorem \ref{trans_disturb} --- that is, the upper bound of the performance difference of the algorithm is related to the size of the transition disturbance. Moreover, since the value functions of all states in these policies are relatively low and the VFR of these policies is low, 
we discover that VPG and PG-CMDP stay robust under transition disturbance since their VFR is low, which is also shown in Theorem~\ref{trans_disturb}. At the same time, we can see that CPPO achieves a higher outcome in different tasks, especially in Swimmer and Hopper. This indicates that our method can improve the robustness of policies under transition disturbance since CPPO controls the risk theoretically related to the robustness under transition disturbance. 

\subsection{Robustness against Observation Disturbance}
\label{experiment_sta}

The agent may also fail in the testing stage because of the gap between its observation and the true state. Consequently, for evaluating the robustness of each algorithm under observation uncertainty, we add a standard Gaussian disturbance to the observation in the testing stage. For this purpose, we plot the performance of the trained policies under observation disturbance in Figure \ref{fig_ob}. The figure shows that the performance degradation is positively related to the size of the disturbance, which is shown in Theorem~\ref{state_disturb}. Similar to the result under transition disturbance, we can also discover that VPG and PG-CMDP stay robust under observation disturbance since their VFR is low, which is shown in Theorem~\ref{state_disturb}. As shown in the figure, CPPO has made significant progress in Swimmer and Hopper compared to baselines. Therefore, CPPO enables us to maintain robustness under observation disturbance, which is also because CPPO controls the risk theoretically related to the robustness under observation disturbance. We also evaluate their robustness under adversarial attacks of state observations and CPPO shows better robustness than other baselines under adversarial attacks. The detailed results are reported in Appendix~\ref{supp_experiment}.

\section{Conclusions}
In this paper, we first provide a theoretical connection between policies' robustness against transition disturbance and observation disturbance, although they are structurally different. Moreover, we analyze the advantages of CVaR for evaluating the uncertainty of policy compared with the worst-case outcome. Based on these analyses, we consider a risk-sensitive optimization objective and propose CPPO to solve it. Extensive experiments on various MuJoCo tasks show that CPPO obtains better performance as well as stronger robustness than various strong competitors.

\section*{Ethical Statement}
Deep reinforcement learning may encounter catastrophic failures due to uncertainty and it is imperative to develop safe reinforcement learning algorithms. This paper studies the robustness of reinforcement learning algorithms against transition and observation disturbance, which is beneficial for safe and reliable reinforcement learning. There are no serious ethics concerns as this is a basic research.


\section*{Acknowledgments}

This work was supported by the National Key Research and Development Program of China (No.s 2020AAA0106000, 2020AAA0104304, 2020AAA0106302), NSFC Projects (Nos. 62061136001, 61621136008, 62076147, U19B2034, U19A2081, U1811461), the major key project of PCL (No. PCL2021A12), Tsinghua-Huawei Joint Research Program, Tsinghua-Alibaba Joint Research Program, a grant from Tsinghua Institute for Guo Qiang, Tsinghua-OPPO Joint Research Center, Beijing Academy of Artificial Intelligence (BAAI), and the NVIDIA NVAIL Program with GPU/DGX Acceleration.



\bibliographystyle{named}
\bibliography{ijcai22}

\newpage

\clearpage 
\appendix

\section{Pseudo Code of CPPO}
\label{pseudecode}
\begin{algorithm}[!htbp]
    \caption{CVaR Proximal Policy Optimization(CPPO)}  
    \label{CPPO}
    \begin{algorithmic} 
        \REQUIRE confidence level $\alpha$ and reward tolerance $\beta$, learning rate $lr_{\eta}, lr_{\theta}, lr_{\lambda}, lr_{\phi}$
        \ENSURE $\theta$ of parameterized policy $\pi_{\theta}$(always be random policy), $\phi$ of parameterized value function $V_{\phi}$.
        \FOR{$k = 1,2,..., N_{iter}$}
        \STATE 
        Generate $N$ trajectories $\mathcal{D}_k = \{\xi_i\}_{i=1}^N$ by following the current policy $\pi_{\theta}$.
        \STATE
        Compute reward $\hat R_i^t$ of each state $s_{i,t}$ in each trajectory $\xi_i$ and the cumulative reward $D(\xi_i)$.
        \STATE
        Compute advantage estimates $\hat A_i^t$ of each state $s_{i,t}$ in each trajectory $\xi_i$.
        \STATE 
        Update parameters respectively:
        \begin{flalign*} 
        \eta \gets \eta -& lr_{\eta} \left(-\lambda + \frac{\lambda}{N(1-\alpha)} \sum_{i=1}^N \textbf{1}\{\eta\ge D(\xi_i)\})\right)
        & \nonumber\\ 
        \theta \gets \theta -&lr_{\theta}\left[\frac{1}{N}\sum_{i=1}^N(\nabla_{\theta}\log P_{\theta}(\xi_i))\frac{\lambda}{1-\alpha}(\eta-D(\xi_i))^+\right.\\
        +  \frac{1}{NT}&\left.\sum_{i=1}^N\sum_{t=0}^T \nabla_{\theta}\min \left( \frac{\pi_{\theta}(a_i^t|s_i^t)}{\pi_{\theta_k}(a_i^t|s_i^t)} \hat A_i^t, g(\epsilon, \hat A_i^t) \right)\right]
        & \nonumber \\ 
        \lambda \gets \lambda +& lr_{\lambda} \left(\frac{\sum_{i=1}^N (\eta-D(\xi_i))^+}{N(1-\alpha)}- \eta + \beta\right)
        & \nonumber \\ 
        \phi \gets \phi +& lr_{\phi}\left(\sum_{i=1}^N\sum_{t=0}^T\frac{ 2(V_{\phi}(s_{i,t}) - \hat R_i^t)\nabla_{\phi}V_{\phi}(s_{i,t})}{NT}\right)
        & \nonumber 
        \end{flalign*}
        \STATE Modify $\beta$ as a function of the return of current trajectories:
        \begin{flalign*}
        \beta\leftarrow l(\xi_1, \xi_2, ..., \xi_N)
        \end{flalign*}
        \ENDFOR
    \end{algorithmic}  
\end{algorithm} 

\newpage

\section{Proofs of Theorems}
\label{appendix-theorems}
In this section, we will provide the proofs of theorems proposed in the paper.

\subsection{The Proof of Theorem \ref{trans_disturb} and Theorem \ref{state_disturb}}
\label{thm_pf_3_4}
Before proving Theorem~\ref{trans_disturb} and Theorem~\ref{state_disturb}, we first examine a property of $d_{\mathcal{M}}^{\pi}$:
\begin{lem}
\label{state_dis_lem}
For any state $s\in \mathcal{S}$, we have:
\begin{equation}
\begin{split}
    &d_{\mathcal{M}}^{\pi}(s) - (1-\gamma)\mathcal{P}(s_0=s)\\
    =& \gamma\sum_{s'}d_{\mathcal{M}}^{\pi}(s')\sum_{a}\pi(a|s')\mathcal{P}(s|s', a).
\end{split}
\end{equation}
\end{lem}
\begin{proof}
By the definition of $d_{\mathcal{M}}^{\pi}(s)$, we have:
\begin{equation}
\begin{split}
&d_{\mathcal{M}}^{\pi}(s) - (1-\gamma)\mathcal{P}(s_0=s) \\
=& (1-\gamma)\sum_{t=1}^{\infty}\sum_{s'} \gamma^t \mathcal{P}(s_{t-1} = s', s_t=s|\pi,\mathcal{M})\\  
=&  (1-\gamma)\sum_{t=0}^{\infty}\sum_{s'} \gamma^{t+1} \mathcal{P}(s_{t} = s'| \pi,\mathcal{M})\mathcal{P}(s_{t+1} = s| s_{t}=s',\pi,\mathcal{M})\\  
=&  \gamma\sum_{s'}\left[(1-\gamma)\sum_{t=0}^{\infty} \gamma^{t} \mathcal{P}(s_{t} = s'| \pi,\mathcal{M})\right]\mathcal{P}(s_{1} = s| s_{0}=s',\pi,\mathcal{M})\\  
=& \gamma\sum_{s'} d_{\mathcal{M}}^{\pi}(s') \mathcal{P}(s_{1} = s| s_{0}=s',\pi,\mathcal{M})\\  
=& \gamma\sum_{s'}d_{\mathcal{M}}^{\pi}(s')\sum_{a}\pi(a|s')\mathcal{P}(s|s', a). 
\end{split}
\end{equation}
Thus we have proven it.   
\end{proof}

Now we will first prove Theorem~\ref{trans_disturb} based on Lemma~\ref{state_dis_lem}.


\begin{proof}
First, considering the bellman equation of value function of $\pi$ in $\mathcal{M}$ and $\hat{\mathcal{M}}$, we have
\begin{equation}
\begin{split}
    V_{\mathcal{M}, \pi}(s) 
    &= \sum_a \pi(a|s)[\mathcal{R}(s, a) + \gamma \sum_{s'}\mathcal{P}(s'|s, a)V_{\mathcal{M}, \pi}(s')],\\
    V_{\hat{\mathcal{M}}, \pi}(s) &= \sum_a \pi(a|s)[\mathcal{R}(s, a) + \gamma \sum_{s'}\hat{\mathcal{P}}(s'|s, a)V_{\hat{\mathcal{M}}, \pi}(s')].
\end{split}
\end{equation}
We define $\Delta V(s) \triangleq V_{\hat{\mathcal{M}}, \pi}(s)-V_{\mathcal{M}, \pi}(s)$ as the difference of these two value functions and can calculate that
\begin{equation}
\begin{split}
\label{22}
    &\Delta V(s)\\
    = &V_{\hat{\mathcal{M}}, \pi}(s)-V_{\mathcal{M}, \pi}(s) \\
    =& \gamma\sum_a \pi(a|s) \sum_{s'}\Delta\mathcal{P}(s'|s,a)V_{\mathcal{M}, \pi}(s')\\
    +& \gamma\sum_a \pi(a|s) \sum_{s'}\hat{\mathcal{P}}(s'|s, a)( V_{\hat{\mathcal{M}}, \pi}(s') - V_{\mathcal{M}, \pi}(s'))\\
    =& \gamma\sum_a \pi(a|s) \sum_{s'}\Delta\mathcal{P}(s'|s,a)V_{\mathcal{M}, \pi}(s')\\
    +& \gamma\sum_a \pi(a|s) \sum_{s'}\hat{\mathcal{P}}(s'|s, a) \Delta V(s'),
\end{split}
\end{equation}
here we set $\Delta\mathcal{P}(s'|s,a) = \hat{\mathcal{P}}(s'|s, a) - \mathcal{P}(s'|s, a)$. Since equation (\ref{22}) satisfies for every state $s$, thus we calculate the expectation of equation (\ref{22}) for $s\sim d_{\mathcal{M}}^{\hat\pi_{\nu}}$ and use Lemma~\ref{state_dis_lem}:
\begin{equation}
\begin{split}
\label{23}
    &\sum_s d_{\hat{\mathcal{M}}}^{\pi}(s) \Delta V(s) \\
    =& \gamma \sum_s d_{\hat{\mathcal{M}}}^{\pi}(s) \sum_a \pi(a|s) \sum_{s'}\Delta\mathcal{P}(s'|s,a) V_{\mathcal{M}, \pi}(s')\\
    +& \gamma\sum_s d_{\hat{\mathcal{M}}}^{\pi}(s) \sum_a \pi(a|s) \sum_{s'}\hat{\mathcal{P}}(s'|s, a) \Delta V(s')\\
    =& \gamma \sum_s d_{\hat{\mathcal{M}}}^{\pi}(s) \sum_a \pi(a|s) \sum_{s'}\Delta\mathcal{P}(s'|s,a) V_{\mathcal{M}, \pi}(s')\\
    +& \sum_{s'}\Delta V(s') \left[\gamma\sum_s d_{\hat{\mathcal{M}}}^{\pi}(s) \sum_a \pi(a|s) \hat{\mathcal{P}}(s'|s, a)\right]\\
    =& \gamma \sum_s d_{\hat{\mathcal{M}}}^{\pi}(s) \sum_a \pi(a|s) \sum_{s'}\Delta\mathcal{P}(s'|s,a)V_{\mathcal{M}, \pi}(s')\\
    +& \sum_{s'}\Delta V(s') \left[d_{\hat{\mathcal{M}}}^{\pi}(s') - (1-\gamma)\mathcal{P}(s_0=s')\right].
\end{split}
\end{equation}
By moving the second term of the right part in (\ref{23}) to the left part, we can deduce that
\begin{equation}
\begin{split}
    &(1-\gamma)\sum_{s'}\Delta V(s') \mathcal{P}(s_0=s')\\
    =& \gamma \sum_s d_{\hat{\mathcal{M}}}^{\pi}(s) \sum_a \pi(a|s) \sum_{s'}\Delta\mathcal{P}(s'|s,a)V_{\mathcal{M}, \pi}(s'),
\end{split}
\end{equation}
thus we have
\begin{equation}
\begin{split}
    &(1-\gamma) (J_{\hat{\mathcal{M}}}(\pi) - J_{\mathcal{M}}(\pi)) \\
    =& (1-\gamma)\sum_{s'}\Delta V(s') \mathcal{P}(s_0=s')\\
    =& \gamma \sum_s d_{\hat{\mathcal{M}}}^{\pi}(s) \sum_a \pi(a|s) \sum_{s'}\Delta\mathcal{P}(s'|s,a) V_{\mathcal{M}, \pi}(s')\\
    =& \gamma \mathbb{E}_{s\sim d_{\hat{\mathcal{M}}}^{\pi}} \mathbb{E}_{a\sim\pi(\cdot|s)}\mathbb{E}_{s'\sim \hat{\mathcal{P}}(\cdot|s, a)} \left(1 - \frac{\mathcal{P}(s'|s, a)}{\hat{\mathcal{P}}(s'|s, a)}\right)V_{\mathcal{M}, \pi}(s').
\end{split}
\end{equation}
Thus we have proven that
\begin{equation}
\begin{split}
    &J_{\hat{\mathcal{M}}}(\pi) - J_{\mathcal{M}}(\pi)\\
    =& \frac{\gamma}{1-\gamma} \mathbb{E}_{s\sim d_{\hat{\mathcal{M}}}^{\pi}} \mathbb{E}_{a\sim\pi}\mathbb{E}_{s'\sim \hat{\mathcal{P}}} \left(1 - \frac{\mathcal{P}(s'|s, a)}{\hat{\mathcal{P}}(s'|s, a)}\right)V_{\mathcal{M}, \pi}(s').
\end{split}
\end{equation}
We take 
\begin{equation*}
\begin{split}
    \bar{V}_{\mathcal{M}, \pi} &= \frac{1}{2}\left( \max_{s'}V_{\mathcal{M}, \pi}(s') + \min_{s'}V_{\mathcal{M}, \pi}(s')\right)  \\
    \hat{V}_{\mathcal{M}, \pi} &=  \max_{s'}V_{\mathcal{M}, \pi}(s') - \min_{s'}V_{\mathcal{M}, \pi}(s')\quad (\text{VFR}). 
\end{split}
\end{equation*}
And we have
\begin{equation*}
\begin{split}
    &|V_{\mathcal{M}, \pi}(s)-\bar{V}_{\mathcal{M}, \pi}|\\
    =& \left|V_{\mathcal{M}, \pi}(s)-\frac{1}{2}\left( \max_{s'}V_{\mathcal{M}, \pi}(s') + \min_{s'}V_{\mathcal{M}, \pi}(s')\right)\right|\\
    \leq& \frac{1}{2}\left| \max_{s'}V_{\mathcal{M}, \pi}(s') - V_{\mathcal{M}, \pi}(s) \right|
    + \frac{1}{2}\left|V_{\mathcal{M}, \pi}(s)- \min_{s'}V_{\mathcal{M}, \pi}(s')\right|\\
    =& \frac{1}{2}\left( \max_{s'}V_{\mathcal{M}, \pi}(s') - \min_{s'}V_{\mathcal{M}, \pi}(s')\right) 
    = \frac{\hat{V}_{\mathcal{M}, \pi}}{2},
\end{split}
\end{equation*}
holds for every state $s$, thus we can prove:
\begin{equation}
\begin{split}
    &|J_{\mathcal{M}}(\pi) - J_{\mathcal{M}}(\hat\pi_{\nu})| \\
    \overset{1}{=} &\frac{\gamma}{1-\gamma} \mathbb{E}_{s\sim d_{\hat{\mathcal{M}}}^{\pi}} \mathbb{E}_{a\sim\pi}\mathbb{E}_{s'\sim \hat{\mathcal{P}}} \left(1 - \frac{\mathcal{P}(s'|s, a)}{\hat{\mathcal{P}}(s'|s, a)}\right) \left(V_{\mathcal{M}, \pi}(s') - \bar{V}_{\mathcal{M}, \pi}\right)\\
    \leq & \frac{\gamma}{1-\gamma} \mathbb{E}_{s\sim d_{\hat{\mathcal{M}}}^{\pi}} \mathbb{E}_{a\sim\pi(\cdot|s)} \mathbb{E}_{s'\sim \hat{\mathcal{P}}(\cdot|s, a)} \left|1 - \frac{\mathcal{P}(s'|s, a)}{\hat{\mathcal{P}}(s'|s, a)}\right|\\
    &\left|V_{\mathcal{M}, \pi}(s') - \bar{V}_{\mathcal{M}, \pi}\right|\\
    \leq & \frac{\gamma}{1-\gamma} \mathbb{E}_{s\sim d_{\hat{\mathcal{M}}}^{\pi}} \mathbb{E}_{a\sim\pi(\cdot|s)} \mathbb{E}_{s'\sim \hat{\mathcal{P}}(\cdot|s, a)} \left|1 - \frac{ \mathcal{P}(s'|s, a)}{\hat{\mathcal{P}}(s'|s, a)}\right| \frac{\hat{V}_{\mathcal{M}, \pi}}{2}\\
    = & \frac{\gamma}{1-\gamma} \mathbb{E}_{s\sim d_{\hat{\mathcal{M}}}^{\pi}} \mathbb{E}_{a\sim\pi(\cdot|s)} \sum_{s'} \left|\hat{\mathcal{P}}(s'|s, a) -  \mathcal{P}(s'|s, a)\right| \frac{\hat{V}_{\mathcal{M}, \pi}}{2}\\
    = & \frac{\gamma}{1-\gamma} \mathbb{E}_{s\sim d_{\hat{\mathcal{M}}}^{\pi}} \mathbb{E}_{a\sim\pi(\cdot|s)} D_{\mathrm{TV}}(\mathcal{P}(\cdot|s,a), \hat{\mathcal{P}}(\cdot|s,a)) \hat{V}_{\mathcal{M}, \pi},
\end{split}
\end{equation}
here the equality 1 holds since $\mathbb{E}_{s'\sim \hat{\mathcal{P}}} \left(1 - \frac{\mathcal{P}(s'|s, a)}{\hat{\mathcal{P}}(s'|s, a)}\right) = 0$ holds for any $s,a$ and $\bar{V}_{\mathcal{M}, \pi}$ is independent of $s'$, which means $\mathbb{E}_{s'\sim \hat{\mathcal{P}}} \left(1 - \frac{\mathcal{P}(s'|s, a)}{\hat{\mathcal{P}}(s'|s, a)}\right)\bar{V}_{\mathcal{M}, \pi} = 0$ holds for any $s,a$. Thus we have proven Theorem~\ref{trans_disturb}. 

\end{proof}


Moreover, we will prove Theorem~\ref{state_disturb} by using the similar method of Theorem~\ref{trans_disturb}.


\begin{proof}
Considering the bellman equation of value function of $\pi,\hat\pi_{\nu}$ in  $\mathcal{M}$, we have:
\begin{equation*}
\begin{split}
    V_{\mathcal{M}, \pi}(s) =& \sum_a \pi(a|s)[\mathcal{R}(s, a) + \gamma \sum_{s'}\mathcal{P}(s'|s, a)V_{\mathcal{M}, \pi}(s')],\\
    V_{\mathcal{M}, \hat\pi_{\nu}}(s) 
    =& \sum_a \pi(a|\nu(s))[\mathcal{R}(s, a) + \gamma \sum_{s'}\mathcal{P}(s'|s, a)V_{\mathcal{M}, \hat\pi_{\nu}}(s')].
\end{split}
\end{equation*}
Similarly, we define $\Delta V(s) \triangleq V_{\mathcal{M}, \hat\pi_{\nu}}(s) - V_{\mathcal{M}, \pi}(s)$ as the difference of these two value functions and we can deduce that
\begin{equation}
\begin{split}
\label{15}
    \Delta V(s) 
    =&V_{\mathcal{M}, \hat\pi_{\nu}}(s) - V_{\mathcal{M}, \pi}(s) \\
    =& \gamma\sum_a \Delta\pi(a|s) \sum_{s'}\mathcal{P}(s'|s, a)V_{\mathcal{M}, \pi}(s')\\
    +& \gamma\sum_a \pi(a|\nu(s)) \sum_{s'}\mathcal{P}(s'|s, a)\Delta V(s')\\
    +& \sum_a \Delta\pi(a|s)\mathcal{R}(s, a),
\end{split}
\end{equation}
here we set $\Delta\pi(a|s)= \pi(a|\nu(s)) - \pi(a|s)$. Since equation (\ref{15}) satisfies for every state $s$, thus we calculate the expectation of equation (\ref{15}) for $s\sim d_{\mathcal{M}}^{\hat\pi_{\nu}}$:
\begin{equation}
\begin{split}
    &\sum_s d_{\mathcal{M}}^{\hat\pi_{\nu}}(s) \Delta V(s)\\
    =&\sum_s d_{\mathcal{M}}^{\hat\pi_{\nu}}(s) [V_{\mathcal{M}, \hat\pi_{\nu}}(s)-V_{\mathcal{M}, \pi}(s)]\\
    =& \gamma \sum_s d_{\mathcal{M}}^{\hat\pi_{\nu}}(s)\sum_a \Delta\pi(a|s) \sum_{s'}\mathcal{P}(s'|s, a)V_{\mathcal{M}, \pi}(s')\\
    +& \gamma\sum_s d_{\mathcal{M}}^{\hat\pi_{\nu}}(s)\sum_a \pi(a|\nu(s))\sum_{s'}\mathcal{P}(s'|s, a)\Delta V(s')\\
    +& \sum_s d_{\mathcal{M}}^{\hat\pi_{\nu}}(s)\sum_a \Delta\pi(a|s) \mathcal{R}(s, a)\\
    =& \gamma \sum_s d_{\mathcal{M}}^{\hat\pi_{\nu}}(s)\sum_a \Delta\pi(a|s) \sum_{s'}\mathcal{P}(s'|s, a)V_{\mathcal{M}, \pi}(s')\\
    +& \sum_{s'}\Delta V(s')\left[\gamma\sum_s d_{\mathcal{M}}^{\hat\pi_{\nu}}(s)\sum_a \pi(a|\nu(s)) \mathcal{P}(s'|s, a)\right]\\
    +& \sum_s d_{\mathcal{M}}^{\hat\pi_{\nu}}(s)\sum_a \Delta\pi(a|s) \mathcal{R}(s, a).
\end{split}
\end{equation}
By applying Lemma~\ref{state_dis_lem}, we have
\begin{equation}
\begin{split}
\label{16}
    &\sum_s d_{\mathcal{M}}^{\hat\pi_{\nu}}(s) \Delta V(s)\\
    =& \gamma \sum_s d_{\mathcal{M}}^{\hat\pi_{\nu}}(s)\sum_a \Delta\pi(a|s) \sum_{s'}\mathcal{P}(s'|s, a)V_{\mathcal{M}, \pi}(s')\\
    +& \sum_{s'}\Delta V(s') \left[d_{\mathcal{M}}^{\hat\pi_{\nu}}(s') - (1-\gamma)\mathcal{P}(s_0=s')\right]\\
    +& \sum_s d_{\mathcal{M}}^{\hat\pi_{\nu}}(s)\sum_a \Delta\pi(a|s) \mathcal{R}(s, a).
\end{split}
\end{equation}
Similarly, by moving the second term of the right part in (\ref{16}) to the left part, we can deduce that
\begin{equation}
\begin{split}
    &(1-\gamma)\sum_{s'}\Delta V(s') \mathcal{P}(s_0=s')\\
    =& \gamma \sum_s d_{\mathcal{M}}^{\hat\pi_{\nu}}(s)\sum_a \Delta\pi(a|s) \sum_{s'}\mathcal{P}(s'|s, a)V_{\mathcal{M}, \pi}(s')\\
    +& \sum_s d_{\mathcal{M}}^{\hat\pi_{\nu}}(s)\sum_a \Delta\pi(a|s) \mathcal{R}(s, a),
\end{split}
\end{equation}
thus we can calculate that
\begin{equation*}
\begin{split}
    &(1-\gamma) (J_{\mathcal{M}}(\hat\pi_{\nu}) - J_{\mathcal{M}}(\pi)) \\
    =& (1-\gamma)\sum_{s'}\Delta V(s') \mathcal{P}(s_0=s')\\
    =& \gamma \sum_s d_{\mathcal{M}}^{\hat\pi_{\nu}}(s)\sum_a \Delta\pi(a|s) \sum_{s'}\mathcal{P}(s'|s, a)V_{\mathcal{M}, \pi}(s')\\
    +& \sum_s d_{\mathcal{M}}^{\hat\pi_{\nu}}(s)\sum_a \Delta\pi(a|s) \mathcal{R}(s, a)\\
    =& \gamma \mathbb{E}_{s\sim d_{\mathcal{M}}^{\hat\pi_{\nu}}} \mathbb{E}_{a\sim\pi(\cdot|\nu(s))} \left(1- \frac{\pi(a|s)}{\pi(a|\nu(s))}\right) \mathbb{E}_{s'\sim \mathcal{P}(\cdot|s, a)} V_{\mathcal{M}, \pi}(s')\\
    +& \mathbb{E}_{s\sim d_{\mathcal{M}}^{\hat\pi_{\nu}}} \mathbb{E}_{a\sim\pi(\cdot|\nu(s))} \left(1- \frac{\pi(a|s)}{\pi(a|\nu(s))}\right) \mathcal{R}(s, a).
\end{split}
\end{equation*}
And we can prove:
\begin{equation}
\begin{split}
    &J_{\mathcal{M}}(\hat\pi_{\nu}) - J_{\mathcal{M}}(\pi) \\
    =& \frac{\gamma}{1-\gamma} \mathbb{E}_{s\sim d_{\mathcal{M}}^{\hat\pi_{\nu}}} \mathbb{E}_{a\sim\pi(\cdot|\nu(s))} \left(1- \frac{\pi(a|s)}{\pi(a|\nu(s))}\right) \mathbb{E}_{s'\sim \mathcal{P}(\cdot|s, a)} V_{\mathcal{M}, \pi}(s')\\
    +& \frac{1}{1-\gamma}\mathbb{E}_{s\sim d_{\mathcal{M}}^{\hat\pi_{\nu}}} \mathbb{E}_{a\sim\pi(\cdot|\nu(s))} \left(1- \frac{\pi(a|s)}{\pi(a|\nu(s))}\right) \mathcal{R}(s, a).
\end{split}
\end{equation}
Similar to the proof of the Theorem 1, we have $ |V_{\mathcal{M}, \pi}(s)-\bar{V}_{\mathcal{M}, \pi}|\leq \frac{\hat{V}_{\mathcal{M}, \pi}}{2}$ holds for every state $s$ and $ \mathbb{E}_{a\sim\pi(\cdot|\nu(s))} \left(1- \frac{\pi(a|s)}{\pi(a|\nu(s))}\right) = 0$ holds. Thus we can prove that
\begin{equation}
\begin{split}
    & |J_{\mathcal{M}}(\pi) - J_{\mathcal{M}}(\hat\pi_{\nu})| \\
    \leq & \frac{\gamma}{1-\gamma} \mathbb{E}_{s\sim d_{\mathcal{M}}^{\hat\pi_{\nu}}} \mathbb{E}_{a\sim\pi(\cdot|\nu(s))} \left| 1- \frac{\pi(a|s)}{\pi(a|\nu(s))}\right|\\
    &\mathbb{E}_{s'\sim \mathcal{P}(\cdot|s, a)} \left|V_{\mathcal{M}, \pi}(s') - \bar{V}_{\mathcal{M}, \pi}\right|\\
    +& \frac{1}{1-\gamma}\mathbb{E}_{s\sim d_{\mathcal{M}}^{\hat\pi_{\nu}}} \mathbb{E}_{a\sim\pi(\cdot|\nu(s))} \left| 1- \frac{\pi(a|s)}{\pi(a|\nu(s))}\right| |\mathcal{R}(s, a)|\\
    \leq & \frac{\gamma}{1-\gamma} \mathbb{E}_{s\sim d_{\mathcal{M}}^{\hat\pi_{\nu}}} \mathbb{E}_{a\sim\pi(\cdot|\nu(s))} \left| 1- \frac{\pi(a|s)}{\pi(a|\nu(s))}\right| \frac{\hat{V}_{\mathcal{M}, \pi}}{2}\\
    +& \frac{1}{1-\gamma}\mathbb{E}_{s\sim d_{\mathcal{M}}^{\hat\pi_{\nu}}} \mathbb{E}_{a\sim\pi(\cdot|\nu(s))} \left| 1- \frac{\pi(a|s)}{\pi(a|\nu(s))}\right|\max_{s,a}|\mathcal{R}(s, a)|\\
    \leq & \frac{\gamma}{1-\gamma} \mathbb{E}_{s\sim d_{\mathcal{M}}^{\hat\pi_{\nu}}} \sum_{a} \left|\pi(a|\nu(s)) -  \pi(a|s)\right| \frac{\hat{V}_{\mathcal{M}, \pi}}{2}\\
    +& \frac{1}{1-\gamma}\mathbb{E}_{s\sim d_{\mathcal{M}}^{\hat\pi_{\nu}}} \sum_{a} \left|\pi(a|\nu(s)) -  \pi(a|s)\right| \max_{s,a}|\mathcal{R}(s, a)|\\
    =& \frac{\gamma}{1-\gamma} \mathbb{E}_{s\sim d_{\mathcal{M}}^{\hat\pi_{\nu}}}\max_s D_{\mathrm{TV}}(\pi(\cdot|s), \pi(\cdot|\nu(s))) \hat{V}_{\mathcal{M}, \pi}\\
    +& \frac{2}{1-\gamma}\mathbb{E}_{s\sim d_{\mathcal{M}}^{\hat\pi_{\nu}}}\max_s D_{\mathrm{TV}}(\pi(\cdot|s), \pi(\cdot|\nu(s)) \max_{s,a}|\mathcal{R}(s, a)|.
\end{split}
\end{equation}
Thus we have proven Theorem~\ref{state_disturb}. 
\end{proof}
Furthermore, we can prove that our bound is tighter than the bound in \cite{Huan2020}:
\begin{equation}
\begin{split}
    &|J_{\mathcal{M}}(\pi) - J_{\mathcal{M}}(\hat\pi_{\nu})| \\
    \leq& \frac{\gamma}{1-\gamma} \max_s D_{\mathrm{TV}}(\pi(\cdot|s), \pi(\cdot|\nu(s))) \hat{V}_{\mathcal{M}, \pi}\\
    +& \frac{2}{1-\gamma} \max_s D_{\mathrm{TV}}(\pi(\cdot|s), \pi(\cdot|\nu(s)) \max_{s,a}|\mathcal{R}(s, a)|\\
    \leq& \frac{2\gamma}{1-\gamma} \max_s D_{\mathrm{TV}}(\pi(\cdot|s), \pi(\cdot|\nu(s))) \max_s |V_{\mathcal{M}, \pi}(s)| \\
    +& \frac{2}{1-\gamma} \max_s D_{\mathrm{TV}}(\pi(\cdot|s), \pi(\cdot|\nu(s)) \max_{s,a}|\mathcal{R}(s, a)|\\
    \leq& \left(\frac{2\gamma}{(1-\gamma)^2} + \frac{2}{1-\gamma} \right) \max_s D_{\mathrm{TV}}(\pi(\cdot|s), \pi(\cdot|\nu(s)))\\
    &\max_{s,a}|\mathcal{R}(s, a)|. 
\end{split}
\end{equation}

\subsection{The Proof of Theorem~\ref{thm-5}}
\label{thm_pf_5}

\begin{proof}
We first calculate $- \mathrm{VaR}_{\alpha}(-V(s))$ as
\begin{equation*}
\begin{split}
    &-\mathrm{VaR}_{\alpha}(-V(s))\\
    =& -\min \{z|F_{V(s)}(-z)\leq 1- \alpha\}\\
    =&-\min \{z| 1 - F_{V(s)}(-z)\ge \alpha\} \\
    =& \max \{-z| 1 - F_{V(s)}(-z)\ge \alpha\}\\ 
    =& \max \{z| F_{V(s)}(z)\leq 1- \alpha\}.
\end{split}
\end{equation*}
Therefore, we have
\begin{equation*}
\begin{split}
      &-\mathrm{CVaR}_{\alpha}(-V(s))\\
      =&  -\mathbb{E}_{z\sim V(s)}\{-z|-z\ge \mathrm{VaR}_{\alpha}(-V(s))\}.\\
      =& \mathbb{E}_{z\sim V(s)}\{z|-z\ge \mathrm{VaR}_{\alpha}(-V(s))\}\\
      =& \mathbb{E}_{z\sim V(s)}\{z| z\leq    -\mathrm{VaR}_{\alpha}(-V(s))\}\\
      =& \mathbb{E}_{\tau}\{D(\tau)| V(s_0)\leq -\mathrm{VaR}_{\alpha}(-V(s))\}\\
      \overset{1}\ge & \mathbb{E}_{\tau}\{D(\tau)| D(\tau)\leq -\mathrm{VaR}_{\alpha}(-D(\tau))\}\\
      =&-\mathrm{CVaR}_{\alpha}(-D(\tau)).
\end{split}
\end{equation*}
Now we will prove the inequality 1. Actually, we set $A = \{\tau = (s_0, a_1, r_1, ...)| V(s_0) \leq -\mathrm{VaR}_{\alpha}(-V(s))\}$ and $B = \{\tau|D(\tau)\leq -\mathrm{VaR}_{\alpha}(-D(\tau))\}$. By the definition of VaR, we have
\begin{equation*}
\begin{split}
    &P(\tau\in A) = P(V(s_0) \leq -\mathrm{VaR}_{\alpha}(-V(s)))\\
    =& P\left(Z \leq -\mathrm{VaR}_{\alpha}(-Z)\right) \\
    = &P(D(\tau)\leq -\mathrm{VaR}_{\alpha}(-D(\tau))) = P(\tau\in B),
\end{split}
\end{equation*} 
here $Z$ can be an arbitrary random variable. Since $B$ includes \textbf{all} trajectories of which the return is lower than $-\mathrm{VaR}_{\alpha}(-D(\tau))$, for any trajectories set $S$ satisfying $P(\tau\in S) = P(\tau\in B)$, we have $\mathbb{E}_{\tau\in S}D(\tau) \ge \mathbb{E}_{\tau\in B}D(\tau)$. Specially, we have $\mathbb{E}_{\tau\in A}D(\tau) \ge \mathbb{E}_{\tau\in B}D(\tau)$ and can prove that
\begin{equation*}
\begin{split}
      & \mathbb{E}_{\tau}\{D(\tau)| V(s_0)\leq -\mathrm{VaR}_{\alpha}(-V(s))\}
      =  \mathbb{E}_{\tau\in A}D(\tau)\\ \ge & \mathbb{E}_{\tau\in B}D(\tau)
      =  \mathbb{E}_{\tau}\{D(\tau)| D(\tau)\leq -\mathrm{VaR}_{\alpha}(-D(\tau))\}.
\end{split}
\end{equation*}
Thus the inequality 1 holds and we have proven Theorem~\ref{thm-5}.
\end{proof}

\subsection{The Proof of Theorem~\ref{lower_bound}}
\label{proof_lower_bound}

\begin{proof}
Since $M$ is the upper bound of the total reward of every trajectory, we have $J(\pi_s) \leq M$. Here we prove this result by considering two scenarios.

In the first case, if $\pi_s$ satisfies that $-\mathrm{CVaR}_{\alpha}(-D(\pi_s))\ge \beta$. Obviously, we have $\pi_c(\alpha, \beta) = \pi_s$, thus 
\begin{equation*}
    J(\pi_c(\alpha, \beta)) = J(\pi_s) \ge \frac{J(\pi_s) - \alpha M}{1-\alpha}.
\end{equation*}
Otherwise, if $\pi_s$ satisfies that $-\mathrm{CVaR}_{\alpha}(-D(\pi_s)) < \beta$, we set $B = -\mathrm{VaR}_{\alpha}(-D(\pi_{c}(\alpha,\beta)))$ and have:
\begin{equation*}
\begin{split}
    &J(\pi_c(\alpha,\beta))\\
    =& \int_{\tau\sim\pi_c(\alpha,\beta)}p(\tau)D(\tau)d\tau\\
    =& \int_{D(\tau)\leq B}p(\tau)D(\tau)d\tau + \int_{D(\tau)>B}p(\tau)D(\tau)d\tau\\
    \ge& \int_{D(\tau)\leq B}p(\tau)D(\tau)d\tau + \int_{D(\tau)>B}p(\tau)B d\tau\\
    =& -\alpha \mathrm{CVaR}(-D(\pi_c(\alpha,\beta))) + B(1-\alpha)\\
    \overset{1}{\ge}& \beta\alpha + B(1-\alpha)\\
    \ge & \beta\alpha + \beta(1-\alpha)\\
    =&\beta,\\
\end{split}
\end{equation*}
here the inequality 1 hold since $-\mathrm{CVaR}_{\alpha}(-D(\pi_c(\alpha,\beta))) \ge \beta$. By the similar way, we set:
\begin{equation*}
\begin{split}
A=&-\mathrm{VaR}_{\alpha}(-D(\pi_{s}))=\max \{z| F_{D(\pi_{s})}(z)\leq 1- \alpha\}, \\
\end{split}
\end{equation*}
thus
\begin{equation*}
\begin{split}
    J(\pi_s) &= \int_{\tau\sim\pi_s}p(\tau)D(\tau)d\tau\\
    &= \int_{D(\tau)\leq A}p(\tau)D(\tau)d\tau + \int_{D(\tau)>A}p(\tau)D(\tau)d\tau\\
    &\leq \int_{D(\tau)\leq A}p(\tau)D(\tau)d\tau + \int_{D(\tau)>A}p(\tau)Md\tau\\
    &= -\mathrm{CVaR}_{\alpha}(-D(\pi_s))(1-\alpha) + M\alpha\\
    &\leq \beta(1-\alpha) + M\alpha\\
    &\leq J(\pi_c(\alpha,\beta))(1-\alpha) + M\alpha.\\
\end{split}
\end{equation*}
So we have proven $J(\pi_c(\alpha,\beta)) \ge \frac{J(\pi_s) - \alpha M}{1-\alpha}$.
\end{proof}

\subsection{The Proof of Equivalently Deforming problem (\ref{raw_problem})}
\label{proof_deform_problem}
In this part, we will equivalently deforming problem (\ref{raw_problem}) as
\begin{equation*}
\begin{split}
&\max_{\theta} J(\pi_{\theta})\quad s.t.-\mathrm{CVaR}_{\alpha}(-D(\pi_{\theta}))\ge \beta\\
\Leftrightarrow& \min_{\theta} -J(\pi_{\theta})\quad s.t.\mathrm{CVaR}_{\alpha}(-D(\pi_{\theta}))\leq -\beta\\
\overset{1}\Leftrightarrow& \min_{\theta} -J(\pi_{\theta})\quad s.t.\min_{\eta\in R}\{\eta + \frac{1}{1-\alpha}\mathbb{E}[(-D(\pi_{\theta})-\eta)^+]\}\leq -\beta\\
\Leftrightarrow& \min_{\theta} -J(\pi_{\theta})\quad s.t.\min_{\eta\in R}\{-\eta + \frac{1}{1-\alpha}\mathbb{E}[(-D(\pi_{\theta})+\eta)^+]\}\leq -\beta\\
\Leftrightarrow& \min_{\theta,\eta} -J(\pi_{\theta})\quad s.t.-\eta + \frac{1}{1-\alpha}\mathbb{E}[(-D(\pi_{\theta})+\eta)^+]\leq -\beta.
\end{split}
\end{equation*}
Here we derive the formula 1 since CVaR owns the property~\cite{risk,CVaRnips2}:
\begin{equation}
    \mathrm{CVaR}_{\alpha}(Z) = \min_{\eta \in R}\left\{ \eta +\frac{1}{1-\alpha}\displaystyle \mathbb{E}[(Z-\eta)^+] \right\}.
\end{equation} 
So we have proven it.

\subsection{Calculating the Gradient of $L(\theta,\eta,\lambda)$}
\label{proof_gradient}

In this part, we will calculate the gradient $\nabla_{\eta}L(\theta,\eta,\lambda),$ $\nabla_{\theta}L(\theta,\eta,\lambda)$ and $\nabla_{\lambda}L(\theta,\eta,\lambda)$ of the function $L(\theta,\eta,\lambda)$ by using the methods in~\cite{CVaRnips}. 
First we can expand the expectation and derive that
\begin{equation*}
\begin{split}
&L(\theta,\eta,\lambda) \\
=& -J(\pi_{\theta})+\lambda \left( -\eta+\frac{1}{1-\alpha}\mathbb{E}[(-D(\pi_{\theta})+\eta)^+]+ \beta\right)\\
=& -\sum_{\xi}P_{\theta}(\xi)D(\xi) +\frac{\lambda}{1-\alpha}\sum_{\xi}P_{\theta}(\xi)(-D(\xi)+\eta)^+\\
&+ \lambda(\beta-\eta).
\end{split}
\end{equation*}
Since $P_{\theta}(\xi)$ only depends on $\theta$ and $\xi$, we have directly calculate the gradient of $\lambda$ as
\begin{equation*}
\begin{split}
    \nabla_{\lambda}L(\theta,\eta,\lambda)
    =&\frac{1}{1-\alpha}\sum_{\xi}P_{\theta}(\xi)(-D(\xi)+\eta)^++ \beta -\eta\\
    =&\frac{1}{1-\alpha}\mathbb{E}_{\xi\sim\pi_{\theta}}(-D(\xi)+\eta)^++ \beta -\eta.
\end{split}
\end{equation*}
Then we calculate the gradient of $\eta$. Since $(-D(\xi)+\eta)^+$ isn't differentiable to $\eta$ at the point of $\eta= D(\xi)$, we consider its semi-gradient as
\begin{equation*}
\begin{split}
& \nabla_{\eta}(-D(\xi)+\eta)^+ =
\left\{
\begin{aligned}
&0\quad\quad\quad\quad\quad\quad \eta < D(\xi)\\
&q\ (0\leq q\leq 1) \quad \eta = D(\xi)\\
&1\quad\quad\quad\quad\quad\quad \eta > D(\xi)\\
\end{aligned}
\right.
\end{split}
\end{equation*}
And we can calculate the gradient of $\eta$ as below:
\begin{equation*}
\begin{split}
&\nabla_{\eta}L(\theta,\eta,\lambda)\\
=& \frac{\lambda}{1-\alpha}\sum_{\xi}P_{\theta}(\xi)\nabla_{\eta}(-D(\xi)+\eta)^+ -\lambda\\
=& \frac{\lambda}{1-\alpha}\sum_{\xi}P_{\theta}(\xi)\textbf{1}\{\eta> D(\xi)\} \\
+& \frac{\lambda q}{1-\alpha}\sum_{\xi}P_{\theta}(\xi)\textbf{1}\{\eta= D(\xi)\} -\lambda\\
=& \frac{\lambda}{1-\alpha}\sum_{\xi}P_{\theta}(\xi)\textbf{1}\{\eta\ge D(\xi)\} -\lambda \quad (\text{set}\ q = 1)\\
=& \frac{\lambda}{1-\alpha}\mathbb{E}_{\xi\sim\pi_{\theta}}\textbf{1}\{\eta\ge D(\xi)\}) -\lambda.
\end{split}
\end{equation*}
Finally, we will calculate the gradient of $\theta$ as
\begin{equation*}
\begin{split}
&\nabla_{\theta}L(\theta,\eta,\lambda) \\
=& -\sum_{\xi}\nabla_{\theta}P_{\theta}(\xi)D(\xi) + \frac{\lambda}{1-\alpha}\sum_{\xi}\nabla_{\theta}P_{\theta}(\xi)(-D(\xi)+\eta)^+\\
=&\sum_{\xi}\nabla_{\theta}P_{\theta}(\xi) \left(-D(\xi) + \frac{\lambda}{1-\alpha}(-D(\xi)+\eta)^+\right)\\
=&\sum_{\xi}(\nabla_{\theta}\log P_{\theta}(\xi))P_{\theta}(\xi)\left(-D(\xi) + \frac{\lambda}{1-\alpha}(-D(\xi)+\eta)^+\right)\\
=& -\mathbb{E}_{\xi\sim\pi_{\theta}}(\nabla_{\theta}\log P_{\theta}(\xi))\left(D(\xi) 
-\frac{\lambda}{1-\alpha}(-D(\xi)+\eta)^+ \right).
\end{split}
\end{equation*}
Thus we have calculated these three gradients.

\begin{figure*}[t]
\centering
\includegraphics[height=3.8cm,width=14cm]{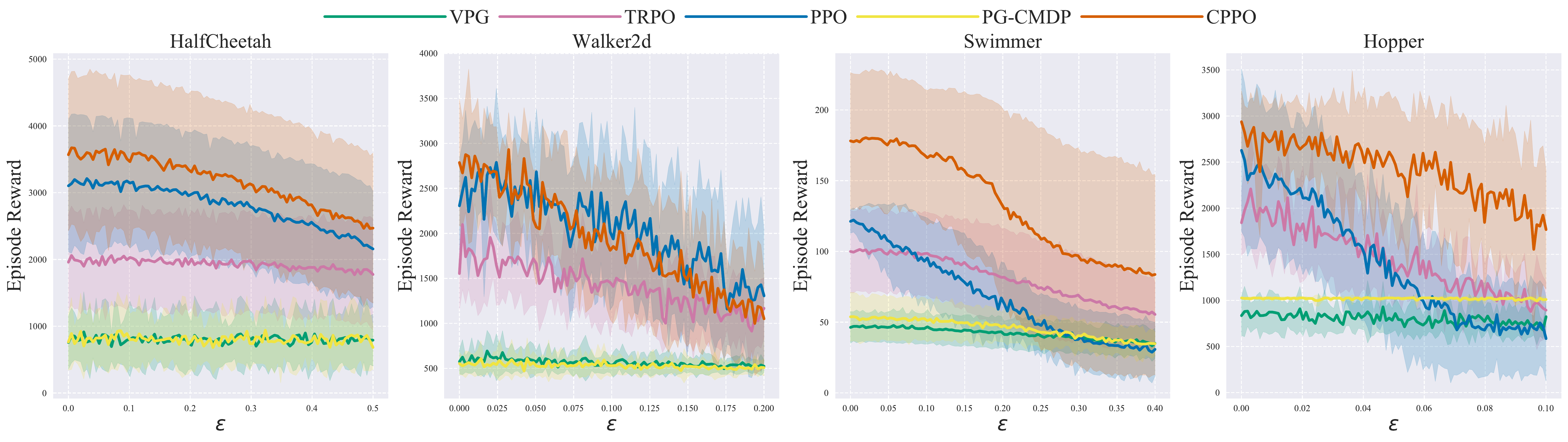}
\caption{Cumulative reward curves for VPG, TRPO, PPO, PG-CMDP and our CPPO under observation adversarial noises. The x-axes indicates the range of the disturbance, and the y-axes indicates the average performance of the algorithm under the state observation adversarial noises.}
\label{fig_ob_advery}
\end{figure*}

\newpage

\section{Supplementary Experiment}
\label{supp_experiment}

In this part, we evaluate the robustness of each algorithm under state observation adversarial disturbance, we calculate one-step FGSM~\cite{goodfellow2014explaining} at each timestep and add it to the observation. Furthermore, we plot the performance of the trained policies under state observation adversarial disturbance in Figure \ref{fig_ob_advery}. As shown in the figure, CPPO has made significant progress in Swimmer, Hopper and HalfCheetah compared to baselines. Therefore, our CPPO enables us to maintain robustness under state observation adversarial disturbance, which is also because our CPPO controls the risk theoretically related to the robustness under observation disturbance.

\end{document}